\documentclass{article}

\usepackage{microtype}
\usepackage{graphicx}
\usepackage{subfigure}
\usepackage{booktabs} 
\usepackage{natbib}  

\usepackage{microtype}
\usepackage{subfigure}
\usepackage{booktabs} 
\usepackage{multicol}
\usepackage{subcaption}

\usepackage{multirow}

\usepackage{amsmath,amsfonts,bm}

\def\eqref#1{equation~\ref{#1}}

\def\1{\bm{1}}

\DeclareMathAlphabet{\mathsfit}{\encodingdefault}{\sfdefault}{m}{sl}
\SetMathAlphabet{\mathsfit}{bold}{\encodingdefault}{\sfdefault}{bx}{n}

\usepackage{wrapfig}
\usepackage{hyperref}
\usepackage{url}

\usepackage{amsmath}
\usepackage{amssymb}
\usepackage{mathtools}
\usepackage{amsthm}

\theoremstyle{plain}
\newtheorem{theorem}{Theorem}[section]

\newtheorem{corollary}[theorem]{Corollary}
\theoremstyle{definition}
\newtheorem{definition}[theorem]{Definition}

\theoremstyle{remark}

\usepackage{booktabs}       
\usepackage{amsfonts}       
\usepackage{nicefrac}       
\usepackage{algorithm}
\usepackage{algorithmic}
\usepackage{times}
\usepackage{graphicx}
\usepackage{epsfig}
\usepackage{amsmath}
\usepackage{amssymb}
\usepackage{caption}
\usepackage{multicol}
\graphicspath{{fig/}}

\usepackage{bbm}

\usepackage{hyperref}

\usepackage[accepted]{icml2021}

\begin{document}

\twocolumn[
\icmltitle{Adaptive Sampling for Deep Learning via Efficient Nonparametric Proxies}

\begin{icmlauthorlist}
\icmlauthor{Shabnam Daghaghi}{to}
\icmlauthor{Benjamin Coleman}{to}
\icmlauthor{Benito Geordie}{goo}
\icmlauthor{Anshumali Shrivastava}{to,goo}

\end{icmlauthorlist}

\icmlaffiliation{to}{Rice University}
\icmlaffiliation{goo}{ThirdAI Corp}

\icmlcorrespondingauthor{Shabnam Daghaghi}{shabnam.daghaghi@rice.edu}

\icmlkeywords{Machine Learning, ICML}

\vskip 0.3in
]



\printAffiliationsAndNotice{}  

\begin{abstract}
Data sampling is an effective method to improve the training speed of neural networks, with recent results demonstrating that it can even break the neural scaling laws. These results critically rely on high-quality scores to estimate the importance of an input to the network. We observe that there are two dominant strategies: \textit{static} sampling, where the scores are determined before training, and \textit{dynamic} sampling, where the scores can depend on the model weights. Static algorithms are computationally inexpensive but less effective than their dynamic counterparts, which can cause end-to-end slowdown due to their need to explicitly compute losses.
To address this problem, we propose a novel sampling distribution based on nonparametric kernel regression that learns an effective importance score as the neural network trains.
However, nonparametric regression models are too computationally expensive to accelerate end-to-end training. 
Therefore, we develop an efficient sketch-based approximation to the Nadaraya-Watson estimator. 
Using recent techniques from high-dimensional statistics and randomized algorithms, we prove that our Nadaraya-Watson sketch approximates the estimator with exponential convergence guarantees. Our sampling algorithm outperforms the baseline in terms of wall-clock time and accuracy on four datasets.
\end{abstract}

\section{Introduction}
\label{intro}
Data volumes have grown exponentially in recent years, causing deep neural networks (DNNs) to become one of the main components of machine learning and artificial intelligence in a diverse range of settings. Recent advancements in complex DNN architectures have pushed the state of the art beyond what was previously thought possible for applications in natural language processing, recommendation systems, and computer vision. Neural scaling laws predict that increased performance can come from dramatic increases in data size, model size, training cost and other parameters~\cite{alabdulmohsinrevisiting}. 

However, the dramatic increase in data scale has created a computational bottleneck in terms of time, energy, and storage. It is costly to train even a simple model on datasets of the scale typically encountered in scientific and industrial settings. Many applications require dedicated, specialized infrastructure to train and run models. Consider a standard click-through prediction task, where a model must predict whether a user will click on an advertisement. Industry research teams report that such tasks can easily reach the scale of a billion events per day~\cite{mcmahan2013ad}. Training a model on the complete dataset is infeasible without considerable resources and expense.

Data selection is a popular approach to handle this problem.
The idea has been independently studied in many contexts. For example,  active learning seeks to define a selection process where data are selectively labeled~\cite{settles2012active}. Coresets and sketches seek to reduce the scale of the data while preserving important metrics -- such as the loss -- within an $\epsilon$ approximation~\cite{phillips2017coresets}. In statistics, a process known as importance sampling can substantially reduce the sample complexity of estimating an unknown quantity.
A sought-after goal of the optimization literature has been to use importance sampling to accelerate SGD~\cite{zhao2015stochastic}.
Recently,~\citet{sorscher2022beyond} demonstrated that data pruning can break the barrier of the neural scaling laws. Their central observation is that neural network training can be significantly accelerated by a sampling process that ranks training data examples by a high-quality ``pruning metric.''

A variety of pruning metrics have been investigated by the community. We observe that these metrics mainly reduce to approximations of the gradient norm as the importance score. This is unsurprising, given that the optimal SGD sampling distribution is known to be proportional to the gradient norm. However, this introduces a problem: the gradient depends on the model parameters. 
We are presented with two options. We may downsample \textit{statically}, scoring each point independently of the network parameters, or \textit{dynamically}, by scoring points according to metrics derived from the current network state. Dynamic sampling naturally results in better accuracy and better iteration-wise convergence. However, these approaches are prohibitively expensive and can degrade the end-to-end performance. 

We seek a way to sample from the subset of high-gradient points at a given training iteration. Fortunately, the gradient norm correlates strongly with the loss, leading to several related approaches. For example, selective backpropagation computes the loss of every point on the full network, but only performs the gradient computation for points with loss exceeding a threshold~\cite{jiang2019accelerating}. Linear regression models have recently been proposed to predict the loss of each point for use in the sampling process, with excellent results~\cite{ganapathiraman2022impon}.
We view these approaches as extremes on a computation-accuracy tradeoff between our ability to estimate the loss and the end-to-end cost of doing so. In this work, we propose a technique that greatly enhances representation capability while reducing cost when compared with forward propagation through the network.

\textbf{Our Contributions:} 

We make the following concrete contributions. 
\begin{enumerate}
    \item We pose the problem of score estimation as a regression task, where we wish to learn a model that assigns a score to each point in the data.
    \item We develop a novel, sketch-based approximation of the Nadaraya-Watson estimator which we call the Nadaraya-Watson sketch (NWS). This sketch may be of independent interest, as it provably approximates the kernel regression model with $O(Nd)$ training and $O(1)$ inference complexity.
    \item Using the NWS, we develop an importance sampling distribution that predicts the loss of the network. By scheduling updates to the NWS, our distribution adapts to the changing network parameters throughout the dynamics of training.
    \item We demonstrate in experiments that our scheme is adaptive and outperforms the baseline in terms of accuracy and wall-clock time on four datasets.
\end{enumerate}

\section{Background}

\setcounter{section}{1}

To develop our proposal, we combine recent ideas from density estimation and randomized algorithms with classical techniques in nonparametric regression. In this section, we provide a brief exposition of the components of our proposal.

\subsection{Nonparametric Regression}

We consider the classical nonparametric regression setting where we are presented with data $\{\mathbf{x}_1, ... \mathbf{x}_N\}$ and outputs $\{y_1,... y_N\}$ generated according to
$$ y_i = f(\mathbf{x}_i) + \epsilon_i$$
where $\epsilon_1, ... \epsilon_N$ are independent residuals with $\mathbb{E}[\epsilon_i] = 0$. We wish to estimate $f$ from the data, which we can do by computing $\mathbb{E}[\mathbf{y} | \mathbf{x}]$ because $\mathbb{E}[y_i | x_i] = \mathbb{E}[f(x_i)] + \mathbb{E}[\epsilon_i] = f(x_i)$. The conditional probability $p(\mathbf{y}|\mathbf{x})$ can be expressed in terms of the joint and marginal probabilities, as follows. 

$$\mathbb{E}[\mathbf{y}|\mathbf{x}] = \int \mathbf{y} \frac{p(\mathbf{x}, \mathbf{y})}{p(\mathbf{x})} d\mathbf{y}$$

The classical Nadaraya–Watson estimator~\cite{nadaraya1964estimating} is obtained by using kernel density estimation to approximate the distributions $p(\mathbf{x}, \mathbf{y})$ and $p(\mathbf{x})$. Given a kernel $k(\mathbf{x}, \mathbf{y})$, we estimate $f$ using a ratio of weighted kernel sums.
\begin{equation}
\label{eq:NW}
\hat{f}(\mathbf{x}) = \frac{\sum_i y_i k(\mathbf{x}, \mathbf{x_i})}{\sum_i k(\mathbf{x}, \mathbf{x_i})}
\end{equation}

The Nadaraya-Watson estimator is known to be pointwise consistent when $E[Y^2] < \infty$ and the kernel satisfies the properties specified by~\citet{greblicki1984distribution}. Specifically, the kernel $k(x, y)$ must have a bandwidth $h$ such that as $N \to \infty$, $h_N \to 0$ and $Nh^d \to \infty$. Stronger guarantees are possible given further assumptions on the problem. For example, if the kernel (or dataset) have compact support then we can attain uniform consistency~\cite{gyorfi2002distribution}.

\subsection{Locality-Sensitive Hashing}
\label{sec:LSH}
We will estimate the numerator and denominator of the Nadaraya-Watson kernel estimator using recent techniques from randomized algorithms for kernel density estimation. These techniques rely on a particular kind of hash function known as a \textit{locality-sensitive hash} (LSH).

\textbf{LSH Functions:} An LSH family $\mathcal{F}$ is a family of functions $l(\mathbf{x}): \mathbb{R}^{d}\to \mathbb{Z}$ that map similar points to same hash value~\cite{indyk1998approximate}. We say that a collision occurs whenever two points have the same hash code, i.e. $l(\mathbf{x}) = l(\mathbf{y})$.

\begin{definition}
\label{def:lsh}
A hash family $\mathcal{F}$ is locality-sensitive with collision probability $k(\cdot,\cdot)$ if for any two points $x$ and $y$, $l(x) = l(y)$ with probability $k(x, y)$
under a uniform random selection of $l(\cdot)$ from $\mathcal{F}$. 
\end{definition}

\textbf{LSH Kernels:} When the collision probability $k(x, y)$ is a monotone decreasing function of the distance metric $\mathrm{dist}(x, y)$, it is well-known that $k$ is a radial kernel function~\cite{coleman2020race}. We say that a kernel function $k(x, y)$ is an \textit{LSH kernel} if it forms the collision probability for an LSH family (i.e. it satisfies the conditions described by~\citet{chierichetti2012preserving}). A number of well-known LSH families induce useful kernels~\cite{gionis1999similarity}.

\subsection{RACE Sketch}
LSH kernels are interesting because there is a family of efficient algorithms based on histograms with randomized partitions to estimate the quantity
$$g(x) = \sum_{x_i\in\mathcal{D}} k(x_i,x)$$
when $k(x_i,x)$ is a hashable kernel~\cite{lei2021fast,ting2021isolation}. Due to the broad utility of kernel sums in statistical estimation, these algorithms have found application in wide-ranging applications such as WiFi localization~\cite{xu2021efficient}, and genomics~\cite{coleman2022one}. However, they all implement the same core method, which we describe here.

We begin by constructing a sketch $S\in \mathbb{Z}^{R\times W}$, a 2D array of integers. Each row of the sketch is indexed using a hash function that assigns a column (or histogram bucket) to an input. This array is sufficient to report an estimate of $g(x)$ for any query $x$. To construct the sketch, we create $R$ independent hash functions $\{h_1, .. h_R\}$ -- one for each row. For each element $x_i \in D$, we increment the corresponding bucket of the sketch. The approximation of $g(x)$ can be done via averaging over the buckets selected by $\{h_1(x), .. h_R(x)\}$~\cite{luo2018arrays} or by using more complex estimation processes such as median-of-means. With the median-of-means estimator, we have the following guarantee~\cite{coleman2021one}.

\begin{theorem}
\label{thm:race_chernoff}
Let $\hat{g}(x)$ be the median-of-means estimate using the RACE sketch with $R$ rows and let $\tilde{g}(x) = \sum_{x_i \in D} \sqrt{k(x_i, x)}$. Then with probability at least $1 - \delta$, 
$$|\hat{g}(x) - g(x)| \leq \left(32\frac{\tilde{g}^2(x)}{R}\log 1/\delta\right)^{1/2}$$
\end{theorem}

\section{Algorithm}

Algorithm~\ref{alg:alg_construct_nws} implements the Nadaraya-Watson estimator via a composition of sketches. We refer to the result as the \textit{Nadaraya-Watson sketch} (NWS). We begin by describing the design of the NWS and prove error bounds on the approximation error. Then, we proceed to describe how to use the sketch as a subroutine of our importance sampling process to accelerate the training of deep learning models.

\begin{algorithm}[!htb]
\caption{Construct NWS}
\begin{algorithmic}[1]
\INPUT Dataset $D = \{(x_i,y_i)\}$, LSH family $\mathcal{F}$, sketch parameters $R$ and $W$
\OUTPUT Sketch $S \in \mathbb{Z}^{R \times W \times 2}$
\STATE Initialize $S_t,S_b \in \mathbb{Z}^{R \times W} = \mathbf{0}$
\STATE Construct $R$ hash functions $H = \{h_1, ... h_R\} \sim \mathcal{F}$
\FOR {$(x_i, y_i) \in D$}
\FOR {$h_r \in H$}
\STATE Increment $S_t[r, h_r(x_i)]$ by $y_i$
\STATE Increment $S_b[r, h_r(x_i)]$ by 1
\ENDFOR
\ENDFOR
\STATE {\textbf{return} $S = [S_t, S_b]$}
\label{alg:alg_construct_nws}
\end{algorithmic}
\end{algorithm}

\subsection{Theory}
In this section, we prove that Algorithm~\ref{alg:alg_construct_nws} produces a sketch that can estimate Equation~\ref{eq:NW} with exponentially-bounded error. Observe that Algorithm~\ref{alg:alg_construct_nws} produces two sketches using the same hash functions. The expected value of the \textit{top} sketch $S_t$ is the numerator of the Nadayara-Watson estimator, while the expected value of the \textit{bottom} sketch $S_b$ is the denominator. We will consider bounds on the ratio $S_t(x) / S_b(x)$.

There are a few subtle design decisions involved with this estimator. First, there are two ways to compute the ratio. One method is to apply the median-of-means process to $S_t$ and $S_b$ independently, and then divide the results. The other way is to perform these steps in reverse order by dividing each row of $S_t$ and $S_b$ and applying median-of-means to the resulting $R$ ratios. We choose to implement the first method because the second one introduces a non-trivial bias term in estimating Equation~\ref{eq:NW}. Second, division by zero can occur whenever $S_b = 0$. However, we observe that when $S_b = 0$, $S_t$ is also $0$ allowing us to correctly return $0$ in this case. Therefore, we exclude this case and consider all expectations in the following analysis to be conditioned on the event that $S_b > 0$ (we omit the notation for the sake of readability). We also suppose that $y$ is bounded. This assumption is standard in the literature and necessary to have bounded variance; see Theorem 3 of~\citet{coleman2020sub}.

\begin{theorem}
\label{thm:nadaraya_watson_sketch}
Let $S_t(x)$ and $S_t(x)$ be the median-of-means estimates over the sketches in Algorithm~\ref{alg:alg_construct_nws} and let $\hat{f}(x)$ be the Nadaraya-Watson estimator. Assuming that $y \in [-B, B]$, we have the following guarantee.

$$ \mathrm{Pr}\left[\left|\frac{S_t}{S_b} - \hat{f}(x)\right| \leq \epsilon\right]\geq 1 - e^{-R \epsilon^2 / 32 B^2 (B + 1 + \epsilon)^2}$$

\end{theorem}
\begin{proof}
Let $g_t(x)$ be the numerator and $g_b(x)$ be the denominator of Equation~\ref{eq:NW}. With $R$ columns, we have the following two guarantees:
$$\mathrm{Pr}[|S_t(x) - g_t(x)| > \epsilon] \leq e^{-R\epsilon^2/32g^2_t(x)}$$
$$\mathrm{Pr}[|S_b(x) - g_b(x)| > \epsilon] \leq e^{-R\epsilon^2 / 32g^2_b(x)}$$

We make two observations. First, note that $S_t(x)$ and $S_b(x)$ can be expressed as the inner products $\langle \mathbf{y}, \mathbbm{1}(x)\rangle$ and $\langle \mathbf{1}, \mathbbm{1}(x)\rangle$, where $\mathbf{y} = [y_1, ... y_N]$ and
$$\mathbbm{1}(x) = \sum_{r = 1}^R [\mathbbm{1}_{\{h_r(x_1) == x\}}, ... \mathbbm{1}_{\{h_r(x_N) == x\}}]$$

Because $S_t(x)$ and $S_b(x)$ are both functions of the same underlying random variable, we do not need to bound the probability for both events. In particular, if $|S_b(x) - g_b(x)| < \epsilon$ and $y_i \in [-B, B]$, then $|S_t(x) - g_t(x)| \leq B|S_b(x) - g_b(x)| < B\epsilon$. Therefore, if we satisfy $|S_b(x) - g_b(x)| < B^{-1}\epsilon$, we will have both $|S_b(x) - g_b(x)| < \epsilon $ and $|S_b(x) - g_b(x)| < \epsilon$. This leads to the following inequality, where we omit the dependence on $x$ for the sake of clarity.

$$-\epsilon<S_t-g_t<\epsilon => -\epsilon+g_t<S_t<\epsilon+g_t$$
$$-\epsilon<S_b-g_b<\epsilon => -\epsilon+g_b<S_b<\epsilon+g_b$$

$$\mathrm{Pr}\left[\frac{g_t - \epsilon}{g_b + \epsilon} \leq \frac{S_t}{S_b} \leq \frac{g_t + \epsilon}{g_b - \epsilon}\right] \leq 1 - e^{-R\epsilon^2/32B^2 g^2_b}$$

To obtain the final inequality, we observe that


$$ \frac{g_t - \epsilon}{g_b + \epsilon} = \frac{g_t}{g_b}  - \epsilon \frac{g_t + g_b}{g^2_b+ \epsilon g_b} \geq \frac{g_t}{g_b} - \epsilon \frac{B+1}{g_b + \epsilon}$$
$$ \frac{g_t + \epsilon}{g_b - \epsilon} = \frac{g_t}{g_b}  + \epsilon \frac{g_t + g_b}{g^2_b- \epsilon g_b} \leq \frac{g_t}{g_b} + \epsilon \frac{B+1}{g_b + \epsilon}$$
where the inequalities follow from $|g_t(x)| \leq B g_b(x)$. 
This leads to
$$ \mathrm{Pr}\left[\left|\frac{S_t}{S_b} - \frac{g_t}{g_b}\right| \leq \epsilon'\right]\geq 1 - e^{-R \epsilon^2 / 32 B^2 g_b^2}, \hspace{2mm}\epsilon'=\epsilon \frac{B+1}{g_b + \epsilon}$$
Replacing $\epsilon = \frac{\epsilon' g_b}{B+1-\epsilon'}$ results in
$$ \mathrm{Pr}\left[\left|\frac{S_t}{S_b} - \hat{f}(x)\right| \leq \epsilon'\right]\geq 1 - e^{-R \epsilon'^2 / 32 B^2 (B + 1 - \epsilon')^2}$$
\end{proof}
Theorem~\ref{thm:nadaraya_watson_sketch} can be used to design a sketch for a given error $\epsilon$ and failure rate $\delta$. Corollary~\ref{cor:design_sketch} demonstrates how to set the parameters to have additive pointwise error with high probability.

\begin{corollary}\label{cor:design_sketch}
The Nadaraya-Watson sketch must have $R = O\left(\frac{B^4}{\epsilon^2}\right) $ rows to have additive error $\epsilon$.
\end{corollary}
\begin{proof}
We require the condition in Theorem~\ref{thm:nadaraya_watson_sketch} to hold with probability $\geq 1 - \delta $. Therefore 
$$ \delta \leq e^{- R \epsilon^2 / 32 B^2 (B + 1 + \epsilon)^2}$$
This implies the following inequalities.
$$ R \epsilon^2 / 32 B^2 (B + 1 + \epsilon)^2 \geq \log 1 / \delta$$
$$ R \geq \frac{32 B^2 (B + 1 + \epsilon)^2}{\epsilon^2} \log 1 / \delta$$
$$ R \geq \frac{32 B^2 (B + 2)^2}{\epsilon^2} \log 1 / \delta$$
where the final inequalities holds under the assumption that $\epsilon < 1$.
\end{proof}

 \subsection{Validation Study}
In this section, our aim is to determine the extent to which the NWS approximates the output of the Nadaraya-Watson kernel regression model. We also demonstrate that the NWS is a reasonable model for regression and classification tasks.
 
\subsubsection{Empirical and Theoretical Error}

Theorem \ref{thm:nadaraya_watson_sketch} suggests that $|\epsilon|\leq O(\frac{1}{\sqrt{R}})$ with high probability. In particular: 
$$\epsilon^2 = \frac{32B^2(B+1-\epsilon)\log\frac{1}{\delta}}{R}\le \frac{32B^2(B+1)\log\frac{1}{\delta}}{R} $$
therefore, we have the following error bound with probability $1 - \delta$: 
$$|\epsilon| \leq B\sqrt{\frac{32 \log{\frac{1}{\delta}}(B+1)}{R}}=O(\frac{1}{\sqrt{R}})$$
To empirically validate this upper bound, we conducted an error study with the Microsoft Research Paraphrase Corpus (MRPC) dataset~\cite{dolan-brockett-2005-automatically}. For a full description of the dataset, see the \textit{Experiments} section. We calculated the ground-truth values of the Nadaraya-Watson kernel model using the training data and computed the error for each sample of the test data. We use the SRP LSH kernel with 10 bits, and we vary the sketch size $R$ to see whether the error obeys our bound. Figure \ref{error_validation} shows the $99\%$ percentile of the empirical error at each value of $R$ (right) and the full distribution of errors (left). These results show that our sketch has the correct asymptotic behavior predicted by our theoretical
results and is bounded by $\frac{1}{\sqrt{R}}$.
\begin{figure}[H]
\begin{center}
\begin{multicols}{2}    \includegraphics[width=1.0\linewidth]{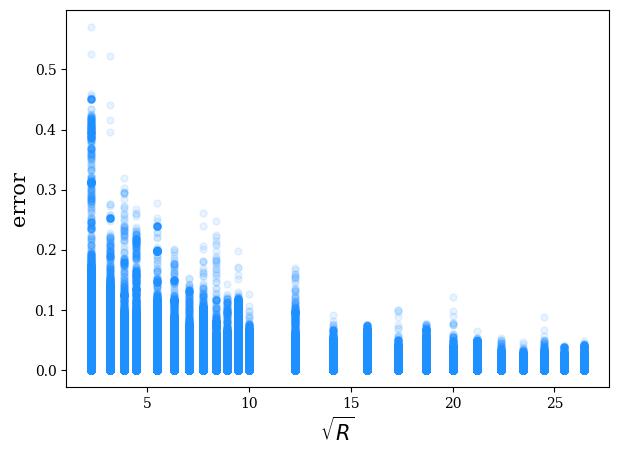}\par 
      \label{}
         \includegraphics[width=1.0\linewidth]{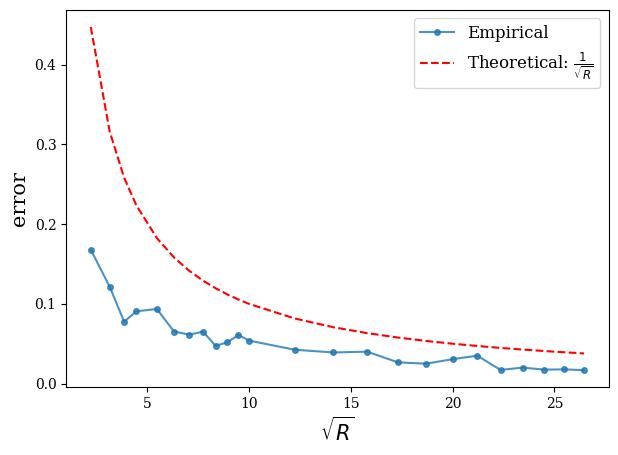}\par 
        \label{} 

    \end{multicols}
\end{center}
\vspace{-0.2in}
\caption{ \textbf{left}: The distribution of empirical error for test dataset for multiple sketches with different values of $R$. \textbf{right:} The blue curve is the $99\%$ percentile of empirical error and the red curve is the theoretical error bound.}
\label{error_validation}
\end{figure}

\subsubsection{NWS for Regression Task}
To demonstrate that the NWS sketch is a useful model, we apply NWS to standard regression datasets. Table \ref{table:nws_lr} shows the comparison of NWS with linear regression on three of the UCI regression datasets, respectively. Note that the performance of the NWS improves as we increase the sketch size $R$, further confirming our theoretical analysis.  

\begin{table}[ht!]
\begin{center}
  \caption{Mean squared error of NWS and linear regression (LR) on UCI regression datasets.} 

\footnotesize

  \begin{tabular}
{p{1.cm}|p{1.cm}|p{.02cm}p{.6cm}p{.6cm}p{.6cm}p{.6cm}p{.5cm}}
    \toprule
   \hline
        \multicolumn{1}{c}{}
      &\multicolumn{1}{c}{LR}
    &\multicolumn{6}{c}{NWS}
    
    \\
    
    Dataset  
    &  &R  & 10 &20   &50  &100  &200    
\\
  
    \\
    \hline 
   airfoli   &15574.41 &   &2251.4   &259.9   & 27.91 &27.76 &27.6  
\\
    \hline
     gas   & 222.97 &   &33.29   &23.36  &18.82  & 18.09  &17.79 
    %
    \\
           \hline
     energy   &9.687 &   &3.015   &1.374    &0.305 &0.0878 &0.078 
    \\
           \hline  
    \bottomrule

\end{tabular}
   \label{table:nws_lr}
\end{center}

\end{table}


\subsection{Adaptive Sampling via the Sketch}

Our validation study demonstrates that the NWS is a reasonable and efficient learning algorithm. In this section, we use the NWS as an online algorithm to predict the importance of an example to the model training process. This is done by fitting the NWS to the sequence of losses observed during training. Because the NWS is a non-linear, non-parametric model, it is able to model the nonconvex loss landscape of the model under training. Our proposed method is a dynamic sampling scheme since it uses the model parameters to estimate the loss, yet it is computationally efficient ($O(1)$) and independent of the number of data points. The proposed method consists of three main steps as shown in Figure \ref{fig:Workflow}. 

\begin{figure*}[ht]
\begin{center}
\centerline{\includegraphics[width=6in]{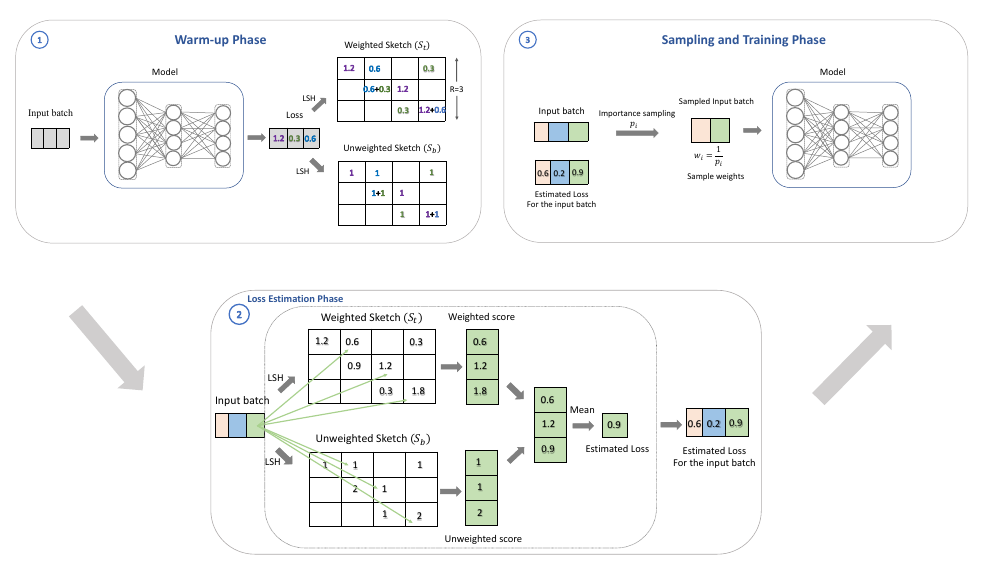}}
\caption{Schematic diagram of our proposal 1) (\textit{Warm-up} phase) For the first few iterations, we add data to the NWS sketch. \textit{Weighted} array stores the loss values, and \textit{unweighted} array stores the number of data. 2) After the \textit{warm-up} phase, we query NWS with the data. The \textit{weighted} score divided by \textit{unweighted} score estimates the loss value for the data point, without any need to explicitly compute the loss with the network. 3) Sampling phase: Based on the estimated loss, we keep the points with higher loss values and reject ones with lower loss values with a higher probability. For more details see Algorithm section.}
\label{fig:Workflow}
\end{center}
\vspace{-5mm}
\end{figure*}

\textbf{Warm-up phase}:  
To initialize the NWS array, we do not down-sample data for the first few iterations.  As Figure \ref{fig:Workflow} represents, in the \textit{warm-up} step, we pass the first few batches of data through the network, compute their loss, and add their loss values to $S_t$ sketch in the numerator of the NWS. The $S_b$ sketch in the denominator of NWS also stores the number of data points. From now on, we call the $S_t$ and $S_b$ sketches, the \textit{weighted} (as it stores the loss values) and \textit{unweighted} sketches, respectively.


\textbf{Loss Estimation phase}: After warm-up phase, we query the NWS with the incoming data batch to retrieve their weighted and unweighted scores. The estimated loss value for each data point is its weighted score divided by the unweighted score. In other words, we are estimating loss via kernel density estimation. 

\textbf{Sampling phase}: We wish to keep samples with higher loss values and discard the ones with lower loss values, since it implies that the network has seen similar data instances. Therefore, we apply importance sampling on the estimated loss values to sample each data point with accepted probability of $p_i$, thus the associated weight of each \emph{accepted} sample is $w_i$ to debias the loss.

\section{Sampling Experiments}
\label{expriment}
In this section, we empirically benchmark the performance of our proposed algorithm against the baseline. The baseline is the conventional training without subsampling, and our proposed algorithm computes the kernel density estimation of loss distribution via NWS and dynamically estimates the loss values for data points. Our algorithm is dynamic and adaptive to the constant change of loss landscape, yet computationally efficient.
We evaluate our framework and the baseline on four datasets with two tasks.

\textbf{Datasets}: MRPC dataset \cite{dolan-brockett-2005-automatically} is an entailment task dataset which consists of a corpus of sentence pairs collected from a news article and each pair is labeled positive if they are paraphrase. Twitter-financial-news and Financial-phrasebank \cite{Malo2014GoodDO} are financial sentiment analysis task datasets. For Financial-phrasebank each sentence is classified from an investor point of view, e.g. how the news may impact the stock price, and for Twitter dataset the finance-related tweets are classified based on their sentiment. Sentinemt140 dataset is also a sentiment analysis task dataset that classifies sentiment of general tweets. 
The statistics of the datasets are shown in Table \ref{table:data}.
\begin{table}[H]
    \centering
\small
    \caption{Statistics of the datasets}

    \begin{tabular}{c|c|c} \hline
    	Dataset  
  & \#Train & \#Test \\ \hline
    	MRPC   & 3669  & 409\\ \hline
            Financial-phrasebank & 4356 & 484 \\ \hline
            Twitter-financial-news & 8944 & 993\\ \hline
    	Sentiment140 & 1.44M  & 1.6M\\ 
    	\hline
    	\bottomrule
    	\end{tabular}
    	\label{table:data}
    	\vspace{-2mm}
\end{table}

 \textbf{Architecture and Hyperparameters}: 
For Sentiment140 dataset we utilize pre-trained Distilled-Bert model \cite{Sanh2019DistilBERTAD} and for the rest of the datasets we utilize the pre-trained Bert model \cite{devlin2018bert}, and add a classifier head to adapt the model to the classification task. We fine-tune the model on each dataset by retraining the whole model. 
The optimizer is Adam with a learning rate of $0.00002$ for all datasets.
To use the hash function we need a vector representation of the data. Therefore, we use the representation of each data in the output of BertPooler layer. 

We use  sign random projection (SRP) hash function with number of repetitions $R=200$ for all datasets. The number of warm-up iterations for MRPC and Financial-phrasebank datasets is 50, and for Sentiment140 and Twitter dataset is 100. We update the NWS sketch with an initial update period of every iteration and then exponentially decay the updating frequency (as we need fewer updates near convergence). Our experiments are run on a NVIDIA V100 GPU with 32 GB memory.

\subsection{Algorithm and Implementation Details}
\label{algo}
We consider NWS sketch which consists of two arrays, one weighted array, and the other unweighted array. The weighted array stores the loss values associated with each sample, while the unweighted array stores the number of points that are mapped to a bucket.


First, we initialize $R$ independent LSH hash functions, where 
$R$ is the number of repetitions in each array. For the sketch to obtain a general idea of the loss landscape, we use the first few iterations to add data to the NWS sketch, with no sampling. We call it the warm-up phase. 
After the warm-up phase, we query both sketches with the incoming batch of data, and compute scores for both arrays (weighted scores and unweighted scores). The final score of each data point is computed as $\frac{\text{weighted score}}{\text{unweighted score}}$, which is equivalent to its estimated loss value.
After calculating the estimated loss for each data point in the batch, we apply importance sampling such that data points with higher estimated loss values are sampled with higher probability. 

We feed the model only the accepted samples, thus the model is trained only on the sampled data points. Then, the true loss values of the sampled data points are calculated and added back to the sketch to update the values.




For more details please refer to Algorithms \ref{alg:alg_main2}, \ref{alg:alg_train}, \ref{alg: alg_loss}, \ref{alg:alg_subsample}.




\begin{algorithm}[!htb]
\caption{Proposed Algorithm}
\begin{algorithmic}[1]
\INPUT Dataset $\mathcal{D}$, Number of warm-up iterations $Iter_{warm}$, \emph{NWS} sketch
\FOR {$Iter$ in Iterations}
\STATE $(x,y) = $ Batch of data $\mathcal{D}$
\IF{$Iter \leq Iter_{warm}$}
\STATE $loss$ = TrainModel($x,y,\_$) (Algorithm \ref{alg:alg_train})
\STATE UpdateSketch($x,y,loss$) {(Algorithm \ref{alg: alg_update})}
\ELSE
\STATE $\hat{loss}$ = LossEstimation(($x,y$), \emph{NWS})
(Algorithm \ref{alg: alg_loss})
\STATE $weight$ = Sampling($x,y,\hat{loss}$) (Algorithm \ref{alg:alg_subsample})
\STATE $loss$ = TrainModel($x,y,weight$) (Algorithm \ref{alg:alg_train})
\STATE UpdateSketch($x,y,loss$) 
\ENDIF
\ENDFOR
\label{alg:alg_main2}
\end{algorithmic}
\end{algorithm}

\begin{algorithm} [!htb]
\caption{TrainModel}
\begin{algorithmic}[1]
\INPUT Batch of data $\mathcal{D} = \{(x,y,weight)\}$
\OUTPUT Loss for samples of each batch $loss$
\STATE If $weight$ is not given: $weight = 1$
\STATE forward run
\STATE Cross Entropy loss for each sample $x_i$: \\$loss_i = CE(x_i,y_i)$ 
\STATE $loss_i = weight_i \cdot loss_i$ 
\STATE backpropagation
\STATE {\textbf{return} $\{loss_i\}$}
\label{alg:alg_train}
\end{algorithmic}
\end{algorithm}

\begin{algorithm}[!htb]
\caption{LossEstimation}
\begin{algorithmic}[1]
\INPUT query $q$ , NWS 

\OUTPUT $S$ set of scores (loss estimation values)
\STATE $S_t$ and $S_b$ are NWS arrays 
\STATE NWS has $\mathcal{R}$ LSH functions $h_{r}$
\STATE{$\text{score}_{\text{weighted}} = \text{Query}(q, S_t, h_{r}|_{k=1}^{k=\mathcal{R}})$ {(Algorithm \ref{alg: alg_query})}}
\STATE{$\text{score}_{\text{unweighted}} = \text{Query}(q, S_b, h_{r}|_{k=1}^{k=\mathcal{R}})$}

\STATE {$\text{S} = \dfrac{\text{score}_{weighted}}{\text{score}_{unweighted}}$}


\STATE {\textbf{return} $S$}

\label{alg: alg_loss}
\end{algorithmic}
\end{algorithm}

\begin{algorithm}[!htb]
\caption{Query}
\begin{algorithmic}[1]

\STATE {\textbf{Input:} Query $q$, sketch $S$, $h_{r}$ as $R$ LSH hash functions} 
\STATE {\textbf{Output:} $score$}
\STATE {Compute query hash codes $h_{r}(q)|_{k=1}^{k=R}$}, map them to buckets $b_r|_{r=1}^{r=R}$ and retrieve the bucket values $x_{b_r}$


\STATE {$score = E[x_{b_r}|_{i=1}^{i=R}]$} \algorithmiccomment{\emph{compute average over the retrieved values}}

\STATE {\textbf{return} $score$}

\label{alg: alg_query}
\end{algorithmic}
\end{algorithm}

\begin{algorithm}[!htb]
\caption{UpdateSketch}
\begin{algorithmic}[1] 

\STATE {\textbf{Input:} Data $\mathcal{D}=(x,y)$, Value $v$, Sketch $S$}
\STATE sketch has $R$ hash functions $H = \{h_1, ... h_R\}$

\FOR {$(x_i, y_i) \in \mathcal{D}$}
\FOR {$h_r \in H$}
\STATE Increment $S[r, h_r(x_i)]$ by $v$
\ENDFOR
\ENDFOR
\label{alg: alg_update}
\end{algorithmic}
\end{algorithm}



\begin{algorithm} 
\caption{Sampling}
\begin{algorithmic}[1]
\INPUT Dataset $\mathcal{D} = \{(x,y)\}$, Estimated loss of each sample $\hat{loss}$
\OUTPUT Sample weight $weight$
\STATE $p_i =$ Accpeted probalilty of sample $i$ via importance sampling over $\hat{loss}$ values 

\IF{$x_i$ is accepted}
\STATE $weight_i = \frac{1}{p_i}$ 
\ELSE 
\STATE $weight_i = 0$
\ENDIF
\STATE {\textbf{return} $\{weight_i\}$}
\label{alg:alg_subsample}
\end{algorithmic}
\end{algorithm}




\subsection{Results}
Table \ref{table:main_table} shows the comparisons in terms of \emph{accuracy} and \emph{convergence time}(wall-clock time to reach baseline accuracy). According to this table, our algorithm meets baseline accuracy faster in terms of wall-clock time (lower convergence time), and eventually reaches higher accuracy level than the baseline for all datasets. 

Figure \ref{fig:main_plot} shows the plots comparing $accuracy$ and $loss$ versus the number of iterations for our method and the baseline. Note that for the first few iterations, the loss and accuracy values are the same for our method and the baseline, this is due to the warm-up phase where we do not subsample and we only update the sketch. 
\begin{figure*}[h!]
\begin{center}
\begin{multicols}{4}
    \includegraphics[width=1.\linewidth]{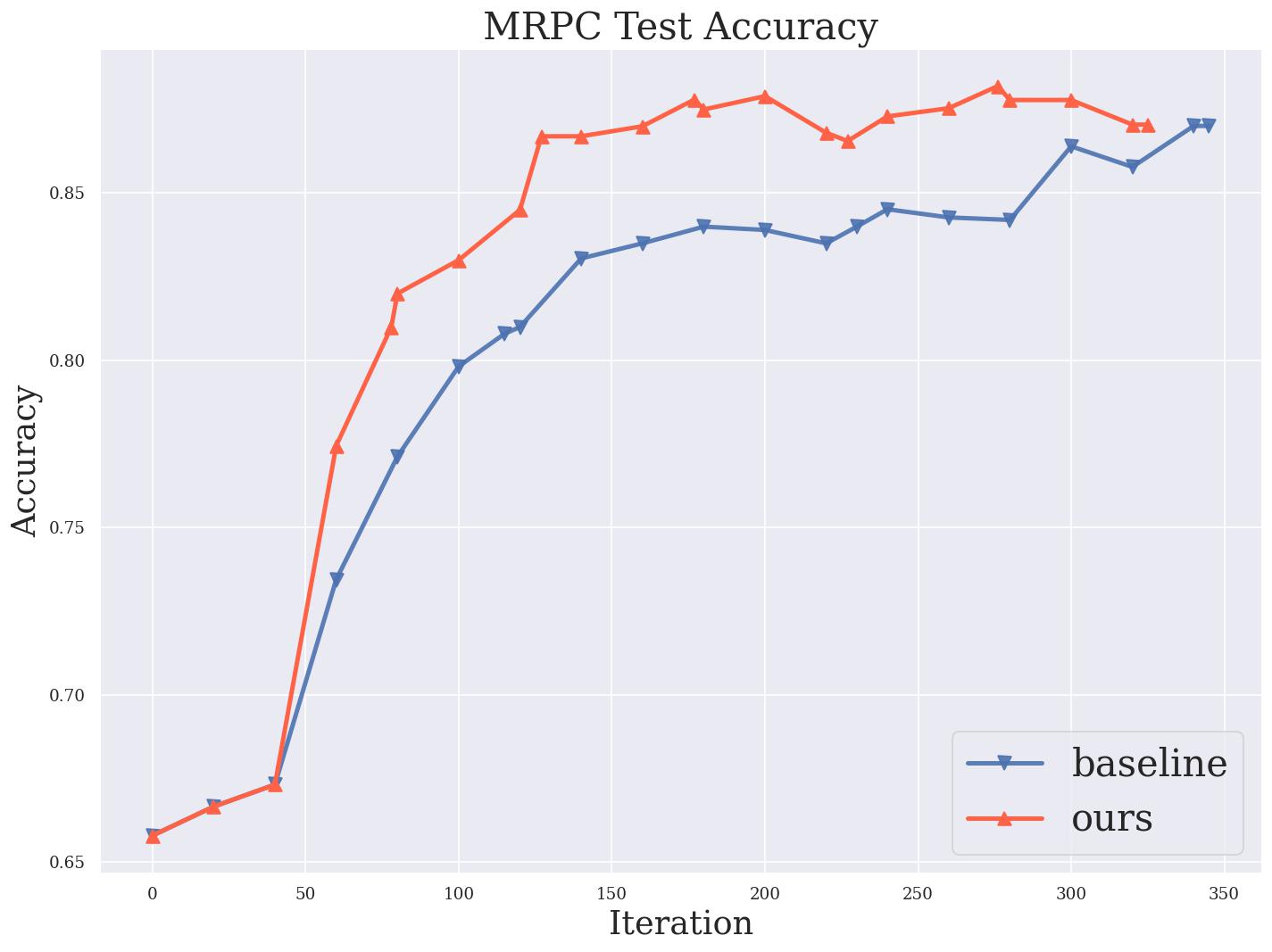}\par 
      \label{}
         \includegraphics[width=1.\linewidth]
     {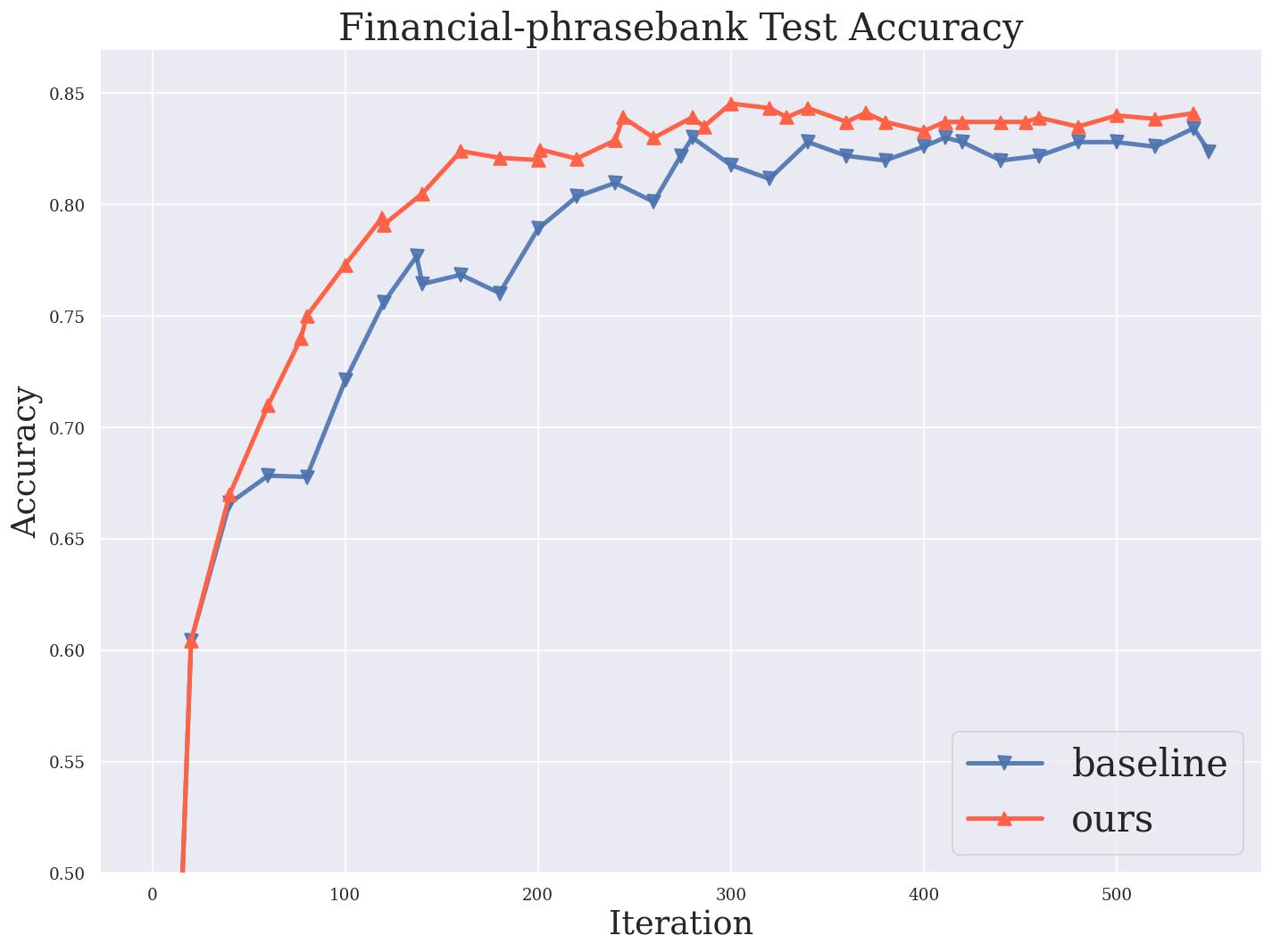}\par
        \label{} 
   \includegraphics[width=1.\linewidth]{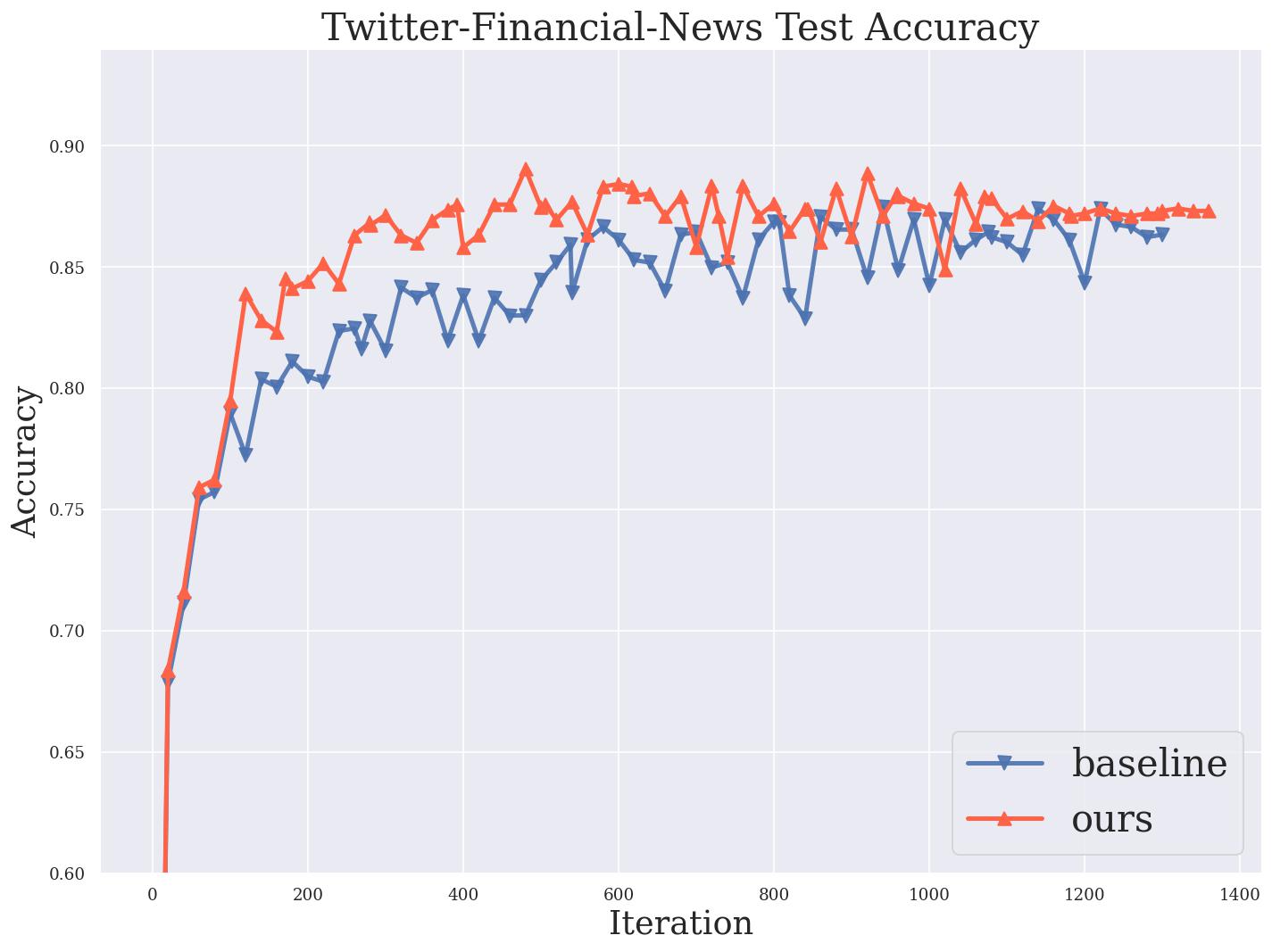}\par

      \label{}
          \includegraphics[width=1.\linewidth]{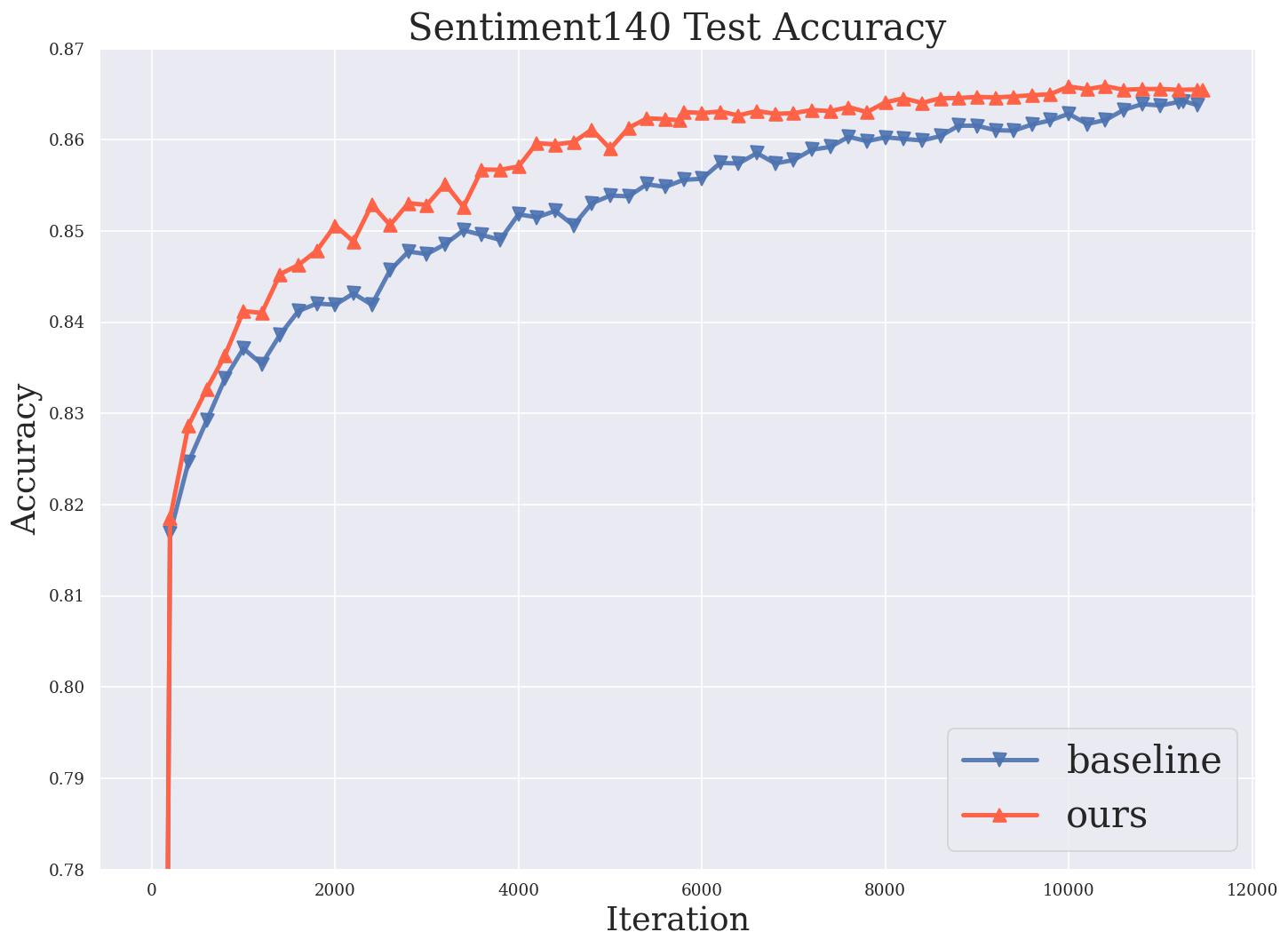}\par 
          \label{}

    \end{multicols}
    \begin{multicols}{4}
    \includegraphics[width=1.\linewidth]{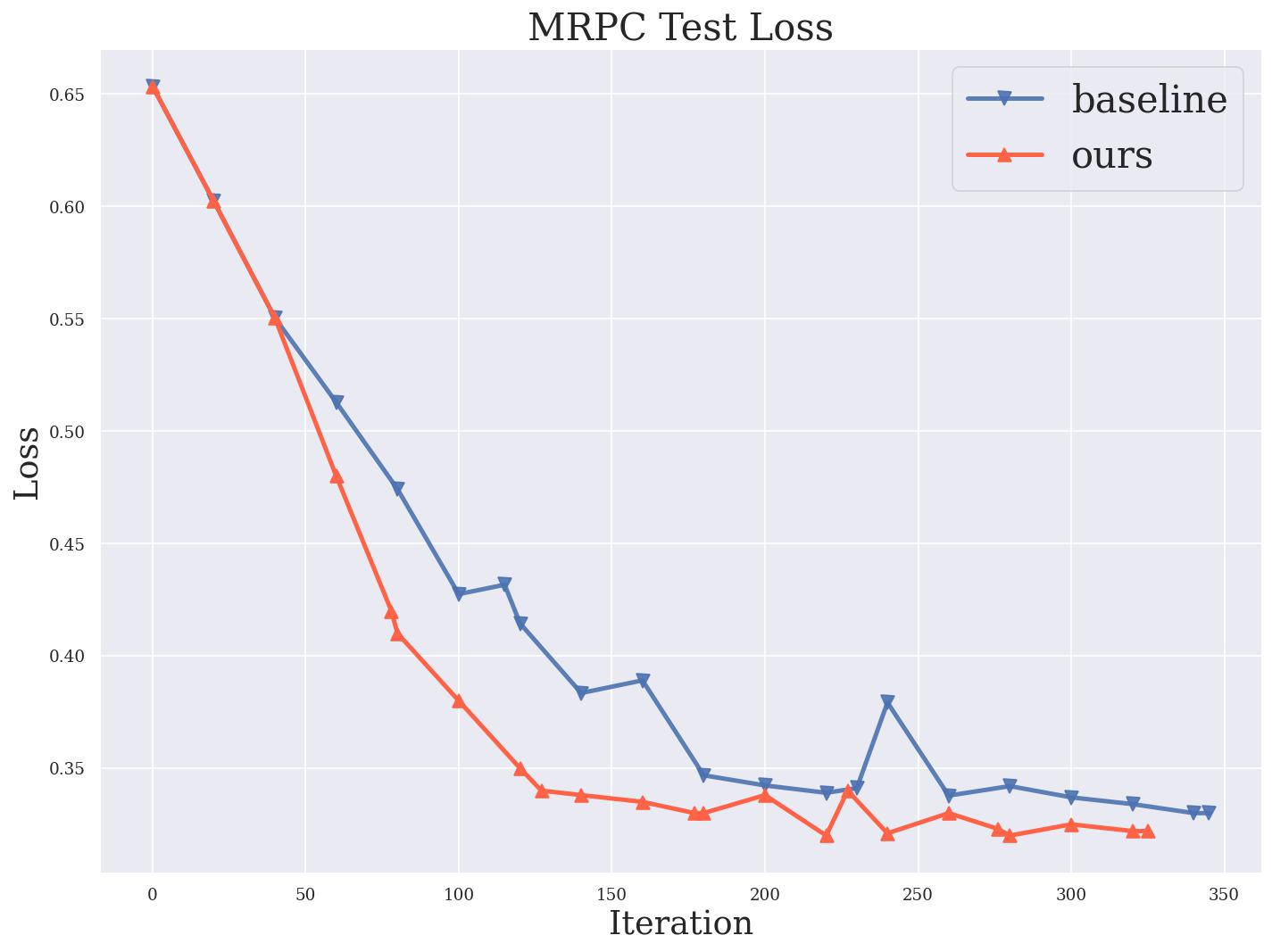}\par

      \label{}
    \includegraphics[width=1.\linewidth]{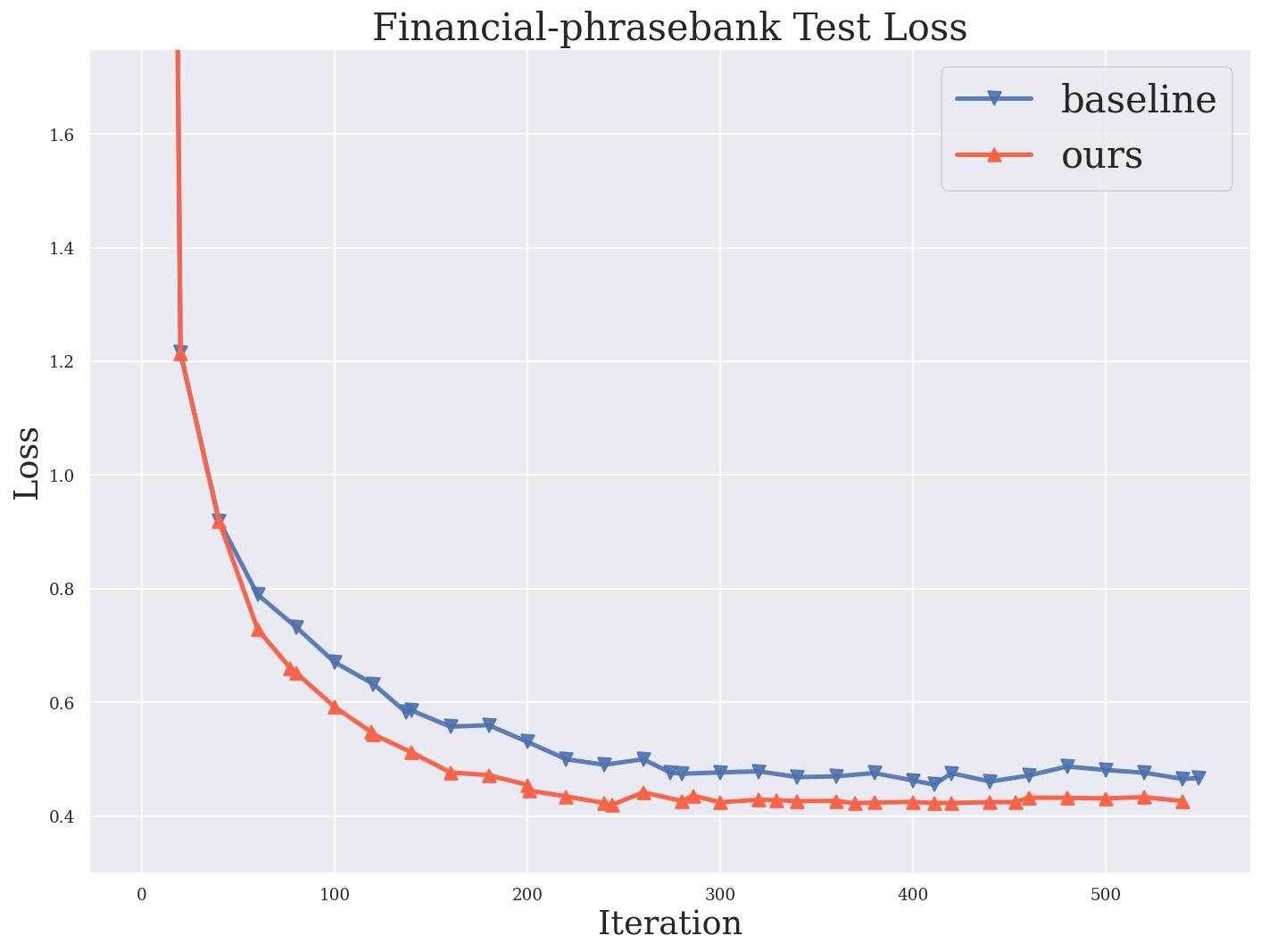}\par 

          \label{}
    \includegraphics[width=1.\linewidth]{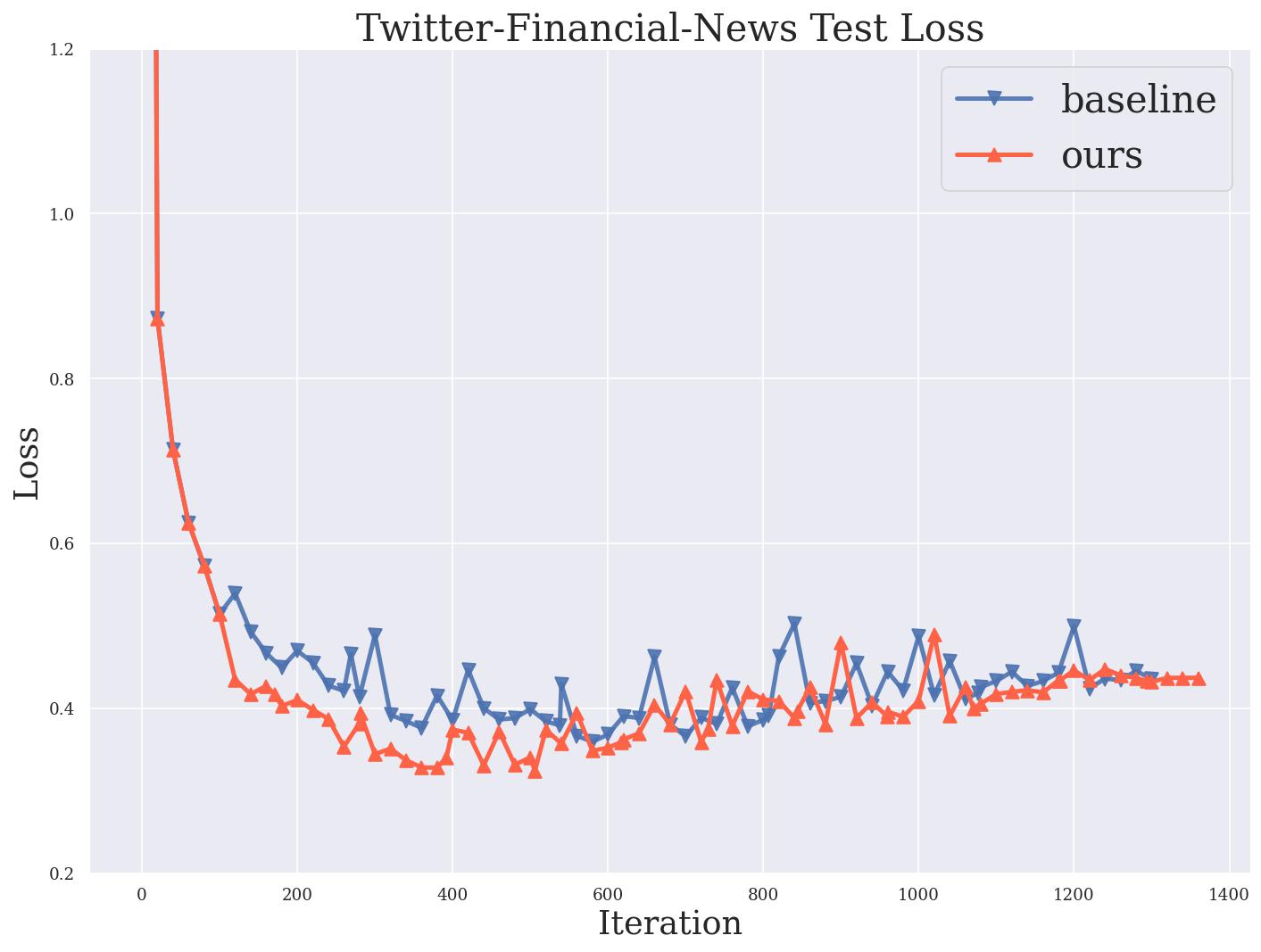}\par 
      \label{}
         \includegraphics[width=1.\linewidth]
         {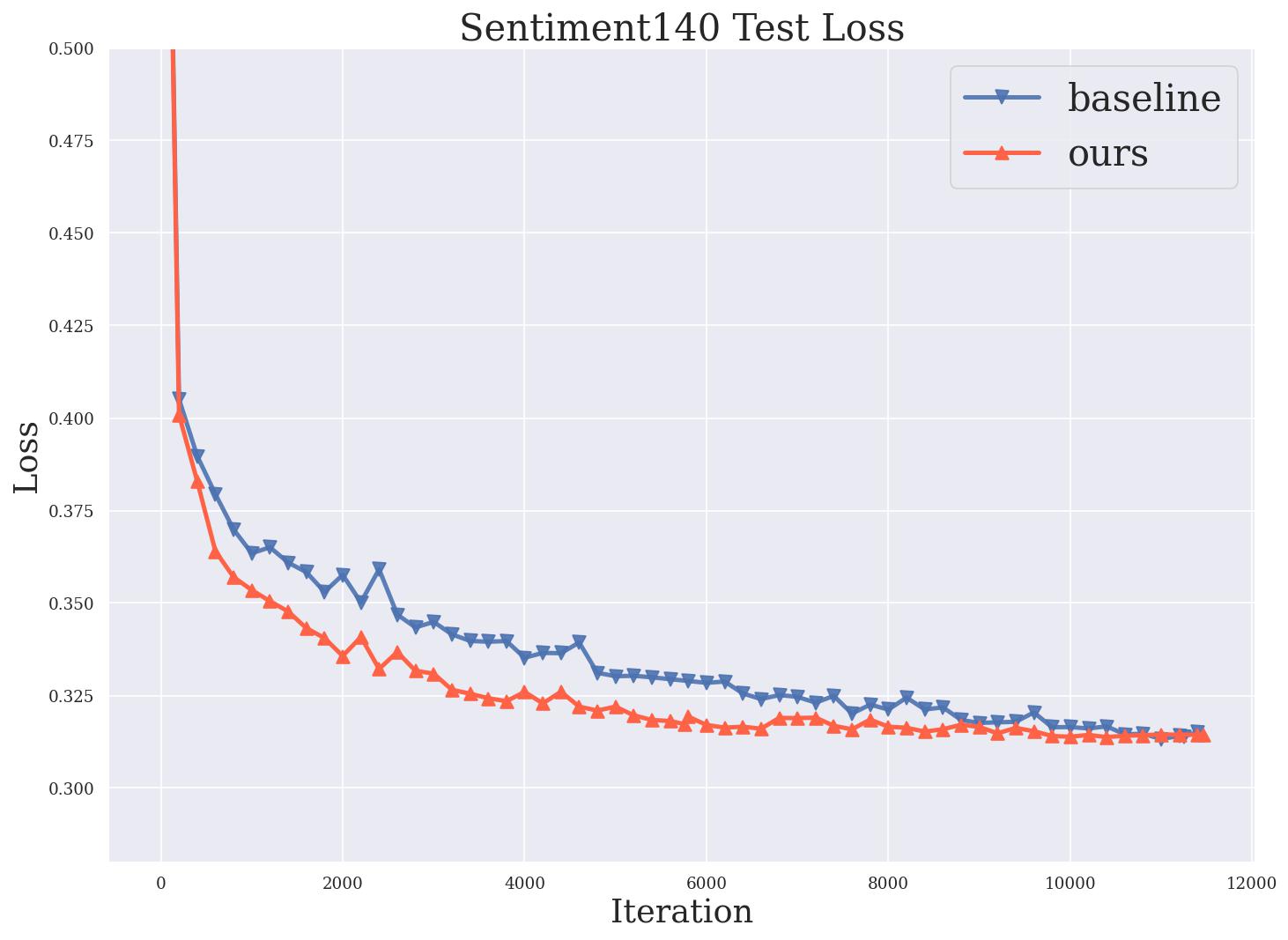}\par 
        \label{} 
    \end{multicols}
\end{center}
\caption{ Comparison of our proposal against the no-sampling baseline for four datasets in terms of loss and accuracy. {\bf Top Row:} represents test accuracy (\emph{y-axis}) vs. number of iterations (\emph{x-axis}) {\bf Bottom Row:} represents test loss (\emph{y-axis}) vs. number of iterations (\emph{x-axis}). The sampling ratio for MRPC, Financial-phrasebank, Twitter and Sentiment140 datasets are 30\%, 30\%, 40\% and 50\%, respectively.}
\label{fig:main_plot}
\end{figure*}

\begin{table*}[ht!]
\begin{center}
  \caption{Comparison of our proposal against the no-sampling baseline w.r.t the final accuracy, and the wall-clock time to reach the final accuracy of the baseline. Our method outperforms the baseline in terms of accuracy and wall-clock time for all datasets.}
\vskip -0.07in

\footnotesize

  \begin{tabular}{p{2.5cm}|p{0.5cm}p{2.4cm}|p{.5cm}p{2.4cm}|p{0.5cm}p{2.4cm}|p{0.5cm}p{2.4cm}}

    \toprule
   \hline
        \multicolumn{1}{c}{}
      &\multicolumn{2}{c}{MRPC}
    &\multicolumn{2}{c}{Financial-phrasebank}
    &\multicolumn{2}{c}{Twitter-financial-news}
    &\multicolumn{2}{c}{Sentiment140}
    \\
    
    Method  &Acc  &wall-clock time \newline to reach baseline Acc
    &Acc   &wall-clock time  \newline to reach baseline Acc
    &Acc  &wall-clock time  \newline to reach baseline Acc
    &Acc  &wall-clock time  \newline to reach baseline Acc
\\
   
    
    \\
    \hline 
   Baseline   &$0.87$ &baseline   &$0.834$ &baseline  &$0.881$ &baseline   &$0.864$ &baseline
\\
    \hline
     Ours    &$0.882$  &1.8x faster  &$0.842$  &1.9x faster  &$0.884$   &1.5x faster   &$0.866$  &1.3x faster
    %
    \\
           \hline
     
     
     
    \bottomrule

\end{tabular}
   \label{table:main_table}
\end{center}

\end{table*}

\section{Related Works}

Sampling and kernel estimation have recently been the focus of a large body of work.\\
\textbf{Kernel Estimators:} The problem of \textit{kernel density estimation} was well-studied in the era of kernelized linear models~\cite{vedaldi2012sparse, chen2012super} and has recently been the focus of intense research due to various reductions of other problems (such as near-neighbor search, graph construction, and kernel matrix multiplication and eigen-decomposition) to density estimation~\cite{coleman2020sub, backurs2019space, siminelakis2019rehashing, backurs2018efficient}.
The NWS bears some resemblance to the RACE kernel density estimator~\cite{coleman2020race, luo2018arrays}. However, there are a few crucial differences between this sketch and prior work. Existing work only considers the density estimation setting, a simpler problem setting where we are interested in approximating a kernel sum. To estimate the Nadaraya-Watson estimator, we must approximate the \textit{ratio} of kernel sums, which is a harder quantity to evaluate. A naive application of the techniques from prior work would result in unbounded variance and an undefined estimator, since the value from the denominator of Equation 1 can become zero. To address this problem, we re-derive the Chernoff bounds for the ratio of (dependent) kernel estimators, noting that the same analysis also produces guarantees for the other hash-based kernel sum approximators.\\
\textbf{Sampling:} There are many works which attempt to improve the speed of training a model by sampling inputs. Elements of the problem have been independently studied in the context of active learning, acceleration of SGD \cite{paul2021deep,johnson2018training}, heuristics to reduce the cost of training large networks, and coresets ~\cite{tukan2021coresets,mirzasoleiman2020coresets}. 
In this review, we distinguish between \textit{static} and \textit{dynamic} methods. Static methods are those that attempt to summarize the dataset without access to the model parameters, while dynamic methods permit access to the parameters as they change during training. Dynamic algorithms typically outperform their static counterparts in terms of sample complexity but incur a higher computational cost.\\

For the comprehensive review of the related work please refer to the supplementary material. 

\section{Conclusion}
We developed a novel sketch-based approximation of the Nadaraya-Watson estimator (NWS) that provably approximates the kernel regression model. Then, we proposed an efficient and dynamic data selection algorithm based on NWS to improve the training of neural networks. Our algorithm utilizes model parameters at each iteration to sample data points with higher loss values, without any explicit computation of loss. We benchmarked our algorithm against no-sampling baseline on four datasets and showed that our proposal outperforms the baseline in terms of accuracy and convergence time.


\nocite{langley00}

\bibliography{example_paper}
\bibliographystyle{icml2021}

\end{document}